\documentclass[letterpaper, 10 pt, conference, twoside]{ieeeconf}
\IEEEoverridecommandlockouts
\usepackage{graphicx}
\usepackage{float}
\usepackage{lscape}
\usepackage{rotating}
\graphicspath{{Figures/}}
\usepackage{algorithm}
\usepackage{algpseudocode}
\usepackage{comment}
\usepackage{regexpatch}
\usepackage{multirow}
\usepackage{multicol}
\usepackage{booktabs}
\usepackage{paralist}
\let\labelindent\relax
\usepackage{enumitem}
\usepackage[flushleft]{threeparttable}

\usepackage{amssymb,amsmath,amsfonts}
\usepackage{array}
\usepackage{bm}
\usepackage{commath}
\usepackage{savesym}
\usepackage[normalem]{ulem}
\usepackage{units}
\usepackage{dsfont}
\usepackage{soul}
\usepackage{xspace}
\usepackage{setspace}
\usepackage{afterpage}
\usepackage{microtype}
\usepackage{silence}
\usepackage[noadjust]{cite}
\makeatletter
\let\NAT@parse\undefined
\makeatother
\usepackage[hidelinks]{hyperref}
\usepackage{cleveref}
\usepackage[all]{hypcap}
\usepackage{balance}
\newcommand{\trs}{\top}

\renewcommand{\b}{\mathbf b}

\newcommand{\p}{\mathbf p}

\newcommand{\w}{\mathbf w}
\newcommand{\x}{\mathbf x}
\newcommand{\z}{\mathbf z}
\newcommand{\EE}{\mathbb E}
\newcommand{\RR}{\mathbb R}
\newcommand{\A}{\mathcal A}
\newcommand{\B}{\mathcal B}
\newcommand{\D}{\mathcal D}
\newcommand{\E}{\mathcal E}

\newcommand{\I}{\mathcal I}
\newcommand{\K}{\mathcal K}
\renewcommand{\L}{\mathcal L}
\newcommand{\M}{\mathcal M}
\newcommand{\N}{\mathcal N}
\renewcommand{\O}{\mathcal O}
\renewcommand{\P}{\mathcal P}

\newcommand{\R}{\mathcal R}
\renewcommand{\S}{\mathcal S}
\newcommand{\T}{\mathcal T}
\newcommand{\V}{\mathcal V}
\newcommand{\X}{\mathcal X}

\newcommand{\mc}[1]{\mathcal{#1}}
\newcommand{\one}{{\mathds 1}}
\newcommand{\given}{{\,\vert\,}}
\newcommand{\lb}{\left(}
\newcommand{\rb}{\right)}
\newcommand{\lbb}{\left[}
\newcommand{\rbb}{\right]}
\newcommand{\lbbb}{\left\{}
\newcommand{\rbbb}{\right\}}
\crefname{equation}{}{}
\renewcommand{\algorithmiccomment}[1]{\bgroup\hfill$\triangleright$~#1\egroup}

\DeclareMathOperator*{\argmax}{arg\,max}

\let\oldforall\forall
\let\forall\undefined
\DeclareMathOperator{\forall}{\oldforall}
\let\oldexists\exists
\let\exists\undefined
\DeclareMathOperator{\exists}{\oldexists}
\usepackage{amsthm}
\theoremstyle{plain}\newtheorem{theorem}{Theorem}
\theoremstyle{remark}\newtheorem{remark}{Remark}[section]
\theoremstyle{definition}
\theoremstyle{definition}\newtheorem{definition}{Definition}

\renewcommand{\figref}[1]{Fig.~\ref{#1}}
\newcommand{\tabref}[1]{Table~\ref{#1}}

\renewcommand{\algref}[1]{Algorithm~\ref{#1}}

\usepackage{subfigure}

\title{Monitoring Over the Long Term:\\Intermittent Deployment and Sensing Strategies for Multi-Robot Teams}
\author{Jun Liu and Ryan K. Williams%
\thanks{This work was supported by the National Institute of Food and Agriculture under Grant 2018-67007-28380.}
\thanks{The authors are with the Department of Electrical and Computer Engineering, Virginia Polytechnic Institute and State University, Blacksburg, VA 24061 USA
(e-mail: \href{mailto:junliu@vt.edu}{junliu@vt.edu}; \href{mailto:rywilli1@vt.edu}{rywilli1@vt.edu}).}}

\begin{document}

\bstctlcite{IEEEexample:BSTcontrol}
\maketitle
\thispagestyle{empty}
\pagestyle{empty}

\begin{abstract}
    In this paper, we formulate and solve the intermittent deployment problem, which yields strategies that couple \emph{when} heterogeneous robots should sense an environmental process, with \emph{where} a deployed team should sense in the environment.  As a motivation, suppose that a spatiotemporal process is slowly evolving and must be monitored by a multi-robot team, e.g., unmanned aerial vehicles monitoring pasturelands in a precision agriculture context.  In such a case, an intermittent deployment strategy is necessary as persistent deployment or monitoring is not cost-efficient for a slowly evolving process.  At the same time, the problem of where to sense once deployed must be solved as process observations yield useful feedback for determining effective future deployment and monitoring decisions. In this context, we model the environmental process to be monitored as a spatiotemporal Gaussian process with mutual information as a criterion to measure our understanding of the environment.
    To make the sensing resource-efficient, we demonstrate how to use matroid constraints to impose a diverse set of homogeneous and heterogeneous constraints. In addition, to reflect the cost-sensitive nature of real-world applications, we apply budgets on the cost of deployed heterogeneous robot teams. To solve the resulting problem, we exploit the theories of submodular optimization and matroids and present a greedy algorithm with bounds on sub-optimality. Finally, Monte Carlo simulations demonstrate the correctness of the proposed method.
\end{abstract}

\section{Introduction}

Multi-robot systems have a natural advantage over a single robot when tasks, especially monitoring, are distributed over space or in time. More robots means large environments can be monitored more effectively, and observations can be made with finer granularity.  These advantages have been seen in various established methods for environmental sensing/modeling \cite{smith2012persistent,lan2013planning,yu2016correlated,zhang2013amcl,liu2013square}.
While these methods are certainly effective, we point out that they are all conditioned on the underlying assumption that \emph{a robotic team should be continuously deployed for a given task}. That is, the majority of multi-robot problems belong to the post-deployment category, meaning we have already made our decision to deploy a multi-robot system for the desired task. Our suggestion in this paper is that in addition to considering the post-deployment problem, we must consider the deployment problem itself since little attention has been given to the concept of deploying multi-robot teams intermittently. Specifically, we want to answer the question ``\emph{when is it the best time to deploy our robots and once deployed where should they sense?}''  In particular, we are motivated by problems from the precision agriculture context where a slowly evolving environmental process must be observed, and thus persistent monitoring is not cost-effective.

\begin{figure}[!t]
    \centering
    \includegraphics[width=.9\linewidth]{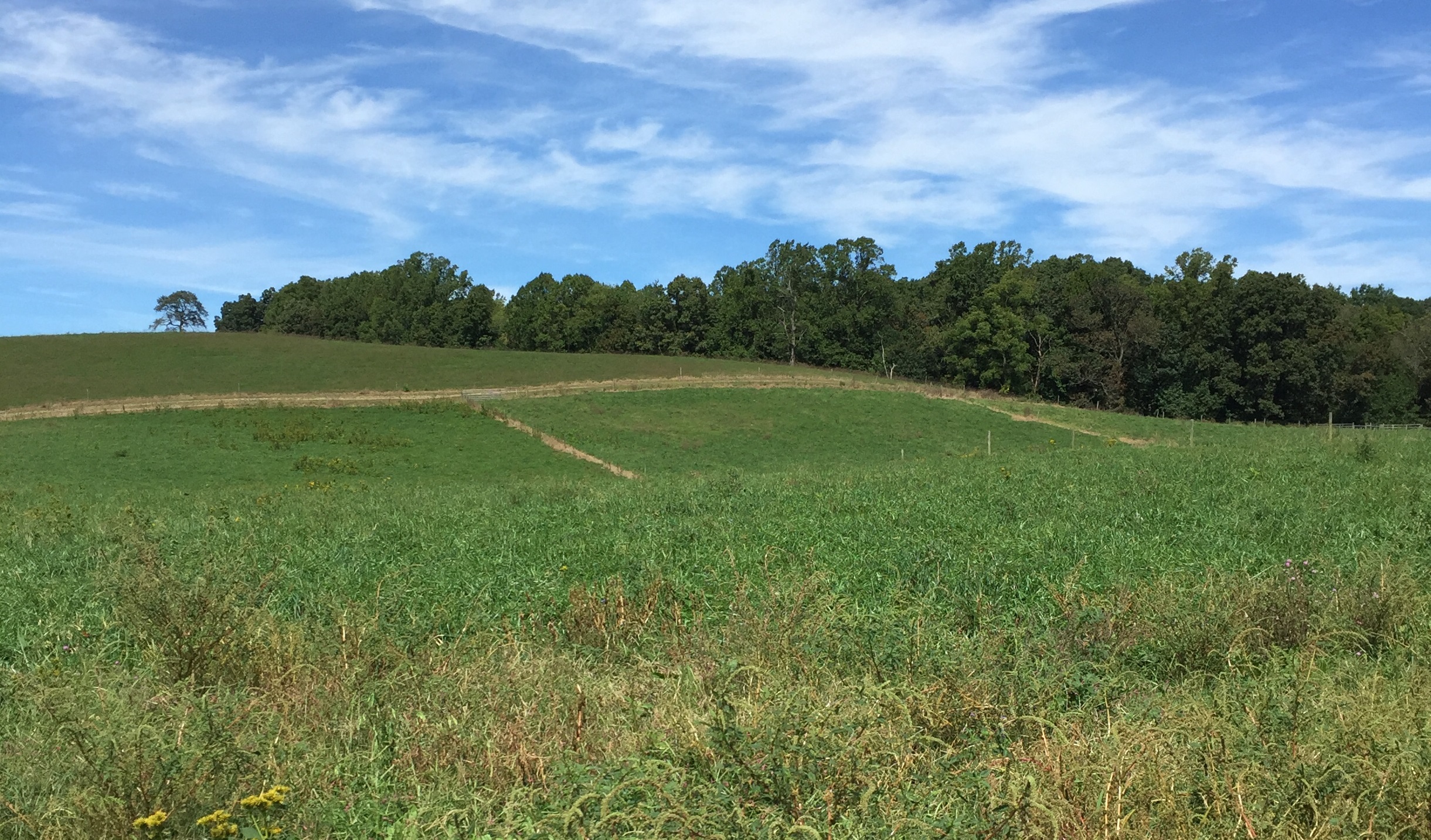}
    \caption{An example of a slowly evolving spatiotemporal process, pastureland forage, that needs to be monitored \cite{das2015devices,franzluebbers2012well,liu2018data,liu2020coupled} for effective land management.}
    \label{fig: pasturelands}
\end{figure}

When our objective is to monitor a slowly evolving environmental process, such as forage in pastureland (\figref{fig: pasturelands}), we cannot deploy robots frequently as the difference in environmental information between deployments may be negligible. Instead, enough time must elapse so that sufficient variation in sensing information allows for an appropriate balance with the cost of deploying a robot team. Indeed, we require an intermittent deployment strategy that defines when it is appropriate to deploy a robotic team, or in other words, how often they should be deployed.  Such a strategy should account for the cost of using differently composed robot teams and, if necessary, should not exceed a cost budget.  At the same time, another problem that couples with the intermittent deployment problem is the sensing location selection problem \cite{krause2008near,joshi2008sensor}. We not only need to know when to deploy our robots, but we also need to decide where these robots should sense once deployed. The reason for this coupling is that effective deployment timing yields a high sensing accuracy without excessive cost while selecting informative sensing locations aids in reducing the deployment needs while maintaining a high sensing accuracy. Since these two problems are naturally coupled, we propose to solve them simultaneously.

A motivating application for this work is precision pastureland management \cite{corson2007evaluating,zhai2006modeling}. Pasturelands are an integral part of agricultural production. To take full advantage of the forage resource, we should monitor pasturelands in a manner that is both cost-effective and sufficiently informative for practitioners. The scale of pasturelands is usually large, and the evolution of forage processes is generally very slow.
Hence, we argue our above-described deployment and sensing strategy is ideally suited for precision pastureland management, as well as problems with similar characteristics such as ocean monitoring \cite{smith2011persistent} or adversarial surveillance \cite{egorov2016target}.

\emph{Related Work:}
The idea of intermittence appears in certain robotic applications. In \cite{sinopoli2004kalman}, the authors demonstrated the convergence of the covariance matrix for Kalman filtering when the observations arrive stochastically. In \cite{hollinger2010multi}, the robots are required to communicate intermittently but with a pre-defined interval before the actual deployment. In contrast, we require the sensing interval to be determined intrinsically by the environment instead of artificially defined a priori. The work in \cite{gini2017multi,melvin2007multi,savelsbergh1985local,mcintire2016iterated} considered time window and precedence constraints in task allocation problems. However, our problem is more complex as an environmental process must be observed. Moreover, our motivation is also contrary to the idea of persistent monitoring \cite{lan2013planning,yu2016correlated}, which is usually not energy efficient (albeit necessary for fast-evolving processes).

Our proposed problem can be considered in the realm of (partially observable) Markov Decision Processes ((PO)MDPs) \cite{puterman2014markov}. Under this framework, \cite{shiryaev1963optimum} proposed an optimal stopping method yielding plans for when to observe an MDP. \cite{lovejoy1987some,krishnamurthy2013schedule} proposed a method for solving the quickest detection/estimation problem with observation uncertainty. Similar work can also be found in \cite{liu2018optimal}. However, as these frameworks require known state transition probabilities and measurement probabilities, their applicability is limited in real-world applications. This paper aims to make the intermittent deployment and sensing problem more general by removing these constraints. Moreover, heterogeneous robot teams and more complex constraints in time are considered in the paper.

Formally, we utilize Gaussian processes \cite{williams2006gaussian} to learn and predict the environment, and use a submodular objective function, mutual information \cite{krause2008near}, to measure the information content we anticipate to gather when deploying robot teams. To reflect the intermittent deployment requirements, we propose to use matroids \cite{oxley2006matroid,liu2019submodular} to impose both homogeneous (all robots) and heterogeneous (robot-specific) constraints.
Additionally, both homogeneous and heterogeneous budget constraints are considered to limit the cost of deployed robot teams. In general, our problem becomes a submodular optimization problem under matroid intersection and knapsack constraints. Generally, submodular maximization is hard even without any constraints \cite{schrijver2003combinatorial}. Multi-linear extension methods \cite{chekuri2014submodular} use a continuous greedy method \cite{vondrak2008optimal} with an approximation bound. However, the running time of these methods is not practical. \cite{badanidiyuru2014fast} proposes a modified greedy method which is an interpolation between the traditional greedy method \cite{fisher1978analysis} and the continuous greedy method \cite{vondrak2008optimal}. It requires less running time to solve the submodular maximization problem with matroid intersection and knapsack constraints. In this paper, we will demonstrate how to utilize this method to solve the intermittent deployment problem.

\emph{Contributions:} In summary, the contributions are as follows:

\begin{enumerate}[label=\arabic*)]
    \item We formalize the intermittent deployment problem for multi-robot systems while considering heterogeneity.
    \item We demonstrate how to utilize matroids to model novel homogeneous and heterogeneous constraints for deployment over time.
    \item We demonstrate how to solve the problem efficiently and provide optimality and complexity analyses.
\end{enumerate}

\section{Preliminaries}
\label{sec: preliminaries}

\subsection{Combinatorial Optimization}
\label{ssec: submodular_matroid}
A set function $f: 2^\V \mapsto \RR$ is a function that maps every subset $\A \subseteq \V$ into a real value $\RR$, and the set $\V$ is usually called the ground set. Related to set functions, we mainly focus on the following properties in this paper.
\begin{definition}[\cite{schrijver2003combinatorial}]
    A set function $f: 2^\V \mapsto \RR$ is
    \begin{itemize}
        \item \emph{normalized}, if $f(\emptyset) = 0$;
        \item \emph{non-decreasing}, if $f(\A) \le f(\B)$ when $\A \subseteq \B \subseteq \V$;
        \item \emph{submodular}, if $f(\A \cup \{e\}) - f(\A) \ge f(\B \cup \{e\}) - f(\B)$ when $\A \subseteq \B \subseteq \V$ and $e \in \V \setminus \B$.
    \end{itemize}
\end{definition}
From the property of submodularity, we can see that the marginal gain of adding an element $e \in \V$ to a smaller set $\A$ is larger than the marginal gain of a larger set $\B$. To simplify, this property can also be written as $f(\{e\} \given \A) \ge f(\{e\} \given \B)$. For this reason, submodular functions can be interpreted as exhibiting diminishing returns. Examples of submodular functions can be found in \cite{schrijver2003combinatorial,liu2019submodular}.

\begin{definition}[\cite{oxley2006matroid}]
    A matroid $\M = (\V, \I)$ is a set system that contains a finite ground set $\V$ and a collection $\I$ of subsets of $\V$. These subsets should satisfy the following properties:
    \begin{enumerate}
        \item $\emptyset \in \I$;
        \item If $\A \subseteq \B \in \I$, then $\A \in \I$;
        \item If $\A, \B \in \I$ and $|\B| < |\A|$, there exists a $e \in \A \setminus \B$ such that $\B \cup \{e\} \in \I$.
    \end{enumerate}
\end{definition}
A matroid $\M = (\V, \I)$ defines the subsets of the ground set $\V$ that are admissible. Thus, if a solution belongs to a matroid, it is one of the admissible sets. Also, the sets contained in $\I$ are called independent sets. A generalization of a single matroid is a matroid intersection. That is, $\I = \{\A \subseteq \V: \A \in \bigcap_{i=1}^p \I_i\}$ is the intersection of $p$ matroids, i.e., $\M_i = (\V, \I_i), \forall i=1, \ldots, p$. The cardinality of this matroid intersection is $|\I| = p$. Note that matroid intersections are not necessarily matroidal \cite{schrijver2003combinatorial}. Examples of matroids and their applications can be found in \cite{schrijver2003combinatorial,oxley2006matroid,jorgensen2017matroid,liu2019submodular,williams2017decentralized}. The reader is referred to \cite{schrijver2003combinatorial} for a comprehensive overview of matroids. A reason of interest for matroids is that it extends the concept of linear independence from linear algebra to set systems and easily represents independence constraints in combinatorial optimization. Importantly, efficient greedy algorithms can be used to optimize submodular objective functions with bounded sub-optimality when the constraint is a matroid intersection \cite{fisher1978analysis}. Thus, we will exploit these tools to solve our problem with simple greedy algorithms and guarantees on the solution quality.

\begin{definition}
    A knapsack constraint is a linear constraint defined through a cost function $c: 2^\V \mapsto \RR$ on the ground set $\V$. That is, $\X = \{ \S \subseteq \V: \sum_{e \in \S} c_e \le B \}$.
\end{definition}

\subsection{Environment Modeling and Sensing}
\label{ssec: Gaussian_process}

Generally, a Gaussian process $f(\x) \sim \mc{GP} (m(\x), \kappa (\x^i, \x^j))$ is specified by a mean function $m(\x) = \EE \lbb f(\x) \rbb$ and a covariance kernel $\kappa \lb \x^i, \x^j \rb = \EE \lbb (f(\x^i) - m(\x^i))(f(\x^j) - m(\x^j)) \rbb$, where $\x^i, \x^j \in \RR^q$. Given a training set $\E = \{ ( \x^i, z^i ) \given i = 1, \ldots, n \}$ with $n$ observations with $\x^i \in \RR^q$ as an input and $z^i \in \RR$ as an output, the properties of $f(\x)$ can be inferred through $\E$. Specifically\footnote{For notation clarity, we drop the parameters of the functions in the remainder of this section, i.e., we use $f$ to represent $f(X)$ and use $f', \mu', \Sigma', m'$ to represent $f(\x'), \mu(\x'), \Sigma(\x') , m(\x')$.}, the goal is to predict $f' \in \RR$ for the testing input $\x' \in \RR^q$ using $\E$. This process is equivalent to calculating $p(f' \given f, X, \z, \x')$, which follows a Gaussian distribution $\N \lb \mu', \Sigma' \rb$ \cite{williams2006gaussian} where $X = [(\x^1)^\trs, \ldots, (\x^n)^\trs]^\trs \in \RR^{n \times q}$ and $\z = [z^1, \ldots, z^n]^\trs \in \RR^{n}$. Also, $\mu' = m' + K_{\x' X} (K_{XX} + \sigma^2_\epsilon I )^{-1} (\z - m(X))$ and $\Sigma' = K_{\x' \x'} - K_{\x' X} (K_{XX} + \sigma^2_\epsilon I)^{-1} K_{X \x'}$, where $\sigma_\epsilon^2 \in \RR$ is the variance of an additive Gaussian noise with zero mean.
$K_{\x'X}$ is the covariance between $\x'$ and $X$ defined by $\kappa(\x', X)$ and similarly for $K_{X \x'}$, $K_{XX}$.  A commonly used kernel function is the squared exponential (SE) kernel $\kappa \lb \x^i, \x^j \rb = \sigma^2 \exp \lb -(2 \ell^2)^{-1} \lb \x^i - \x^j \rb ^2 \rb$, where $\sigma^2$ and $\ell$ are the variance and length-scale respectively.

If $\V$ is the ground set containing all available choices for collecting data for a GP, one criterion for measuring the quality of the collected information contained in dataset $\D$ is the mutual information \cite{cover2012elements} between $\D$ and $\V \setminus \D$. That is, $ M(\D) = H(\D) - H(\D \given \V \setminus \D)$,
where $H(\D) = \frac{1}{2} \log \det \lb 2 \pi e \Sigma(\D) \rb$ is the entropy and $\Sigma(\D)$ is the covariance of $\D$. Mutual information is symmetric, i.e., $M(\D) = M(\V \setminus \D)$. Using the relationship between mutual information and entropy, we get the marginal gain \cite{krause2008near} of element $e$ under current set $\D$ as $M(\{e\} \given \D) = M(\{e\} \cup \D) - M(\D)
    = H(\{e\} \given \D) - H(\{e\} \given \V \setminus (\D \cup \{e\}))$.

In this work, we will exploit the metric of mutual information on information collected over space and in time during intermittent deployments.  The addition of the time dimension leaves all of the typical properties and evaluation methodologies of mutual information intact, such as in \cite{krause2008near}. Thus we omit those details here as it is not our novelty.

\section{The Intermittent Deployment Problem}
\label{sec: problem_formulation}

\subsection{General Modeling}
\label{ssec: problem_formulation}

Consider a spatiotemporal Gaussian process $f(\x) \sim \mc{GP} ( m(\x), \kappa ( \x^i, \x^j))$ evolving over a $P \times Q$ 2D field $\P$ with time indexed by $t$ where $\x^i = [x^i, y^i, t]^\trs \in \RR^3$. $\x^i$ is composed of a 2D location $\p^i = (x^i, y^i)$ and the corresponding time $t$. Also, $z^i(\x^i) \in \RR$ is the corresponding measurement. $\P = \{1, \ldots, P \times Q\}$ contains the indexes of the available locations for sensing the process. The initial training data set is $\E = \{ \lb \x^{i}, z^{i} \rb \given i = 1, \ldots, n \}$ with $n$ observations, where $X = [(\x^{1})^\trs, \ldots, (\x^{n})^\trs]^\trs \in \RR^{n \times 3}$ is the input and $\z = [z^{1}, \ldots, z^{n}]^\trs \in \RR^{n}$ is the output. Since $\E$ is only used for training an initial environment model, it can derive either from existing process data (e.g., from similar pastureland deployments) or from manual short-term deployments in the target environment before intermittent deployment begins\footnote{It is worth noting that our formulation also applies if the initial process model is poor (high uncertainty).  In this case, initial deployment decisions may be poor.  However as data is collected and the model improves, deployment decisions will correspondingly improve.}. For the kernel function, we use a separable spatiotemporal kernel to build a composite covariance function as $\kappa \lb \x^i, \x^j \rb = \kappa_p \lb \p^i, \p^j \rb \cdot \kappa_t \lb t^i, t^j \rb + \sigma_\epsilon^2$,
where $\kappa_p$ and $\kappa_t$ are the kernels for the $\p$ and $t$ dimensions, respectively.

Let us now build the ground set $\V$, which contains all possible deployment decisions for a heterogeneous multi-robot team over space and in time.
We denote by $\T = \{1, \ldots, T\}$ the set of time indices when robots can be deployed. Consider a set of heterogeneous robots $\R$ which are available to deploy over the time horizon $\mathcal{T}$. We denote by $\R = \{1, \ldots, R\}$ the set of indices of all robots. Each robot $r \in \R$ has a different sensing capability, i.e., different sensing noise variance $\sigma_r^2$, and different deployment cost $c_r^i$
with respect to different location index $i \in \P$. Every location index $i$ is associated with a location $(x^i, y^i)$.
Also, when considering the heterogeneity of robot $r$ at different time $t$, we can write the above sensing noise variance and cost as $\sigma_{r}^2(t)$ and $c_r^i(t)$.
Now, our ground set per time contains all available choices at time $t$ in the Cartesian product of $\P$ and $\R$ denoted by $\V_t = \{(x_r^i, y_r^i, t): \forall i \in \P, r \in \R\}$.
We interpret $(x_r^i, y_r^i, t)$ as ``location $(x^i, y^i)$ is sensed by robot $r$ at time $t$''. Finally, the ground set of the intermittent deployment problem is $\V = \bigcup_{t \in \T} \V_t$. Importantly, the $\V_i$'s are disjoint, i.e., $\V_i \cap \V_j = \emptyset, \forall i, j \in \T$. The cardinality of the ground set is $|\V| = P \cdot Q \cdot R \cdot T$. In \tabref{tab: ground_set}, we show all elements in the ground set $\V$, where we use $N = P \cdot Q$ for brevity.

\begin{table}[!tbp]
    \centering
    \begin{threeparttable}
        \caption{The ground set of the intermittent deployment problem}
        \label{tab: ground_set}
        \begin{tabular}{cccc}
            \toprule
                                         & $t=1$               & $\ldots$ & $t=T$               \\ \midrule
            \multirow{3}{*}{$r = \{1\}$} & $(x_1^1, y_1^1, 1)$ & $\ldots$ & $(x_1^1, y_1^1, T)$ \\
                                         & $\ldots$            & $\ldots$ & $\ldots$            \\
                                         & $(x_1^N, y_1^N, 1)$ & $\ldots$ & $(x_1^N, y_1^N, T)$ \\ \hline
            \multirow{3}{*}{$r = \{r\}$} & $(x_r^1, y_r^1, 1)$ & $\ldots$ & $(x_r^1, y_r^1, T)$ \\
                                         & $\ldots$            & $\ldots$ & $\ldots$            \\
                                         & $(x_r^N, y_r^N, 1)$ & $\ldots$ & $(x_r^N, y_r^N, T)$ \\ \hline
            \multirow{3}{*}{$r = \{R\}$} & $(x_R^1, y_R^1, 1)$ & $\ldots$ & $(x_R^1, y_R^1, T)$ \\
                                         & $\ldots$            & $\ldots$ & $\ldots$            \\
                                         & $(x_R^N, y_R^N, 1)$ & $\ldots$ & $(x_R^N, y_R^N, T)$ \\ \bottomrule
        \end{tabular}
    \end{threeparttable}
\end{table}

\subsection{Constraint Modeling}
\label{ssec: constraint_modeling}

\noindent
1) \emph{Knapsack constraints}.

The general form is $\X = \{\S \subseteq \V: c(S) \le B\}$, where $c: 2^\V \mapsto \RR$ is a cost function.
This form can be used in the following ways:
\begin{itemize}
    \item $\X_1 = \{\S \subseteq \V: c(\S_r) \le B, \forall r \in  \R \}$,
          where $S_r = \{(x_k^i, y_k^i, t) \given k = r\}$.
          In this case, we set a constraint $B$ on the cost used by each robot.

    \item $\X_2 = \{\S \subseteq \V:  c(\S(t)) \le B, \forall t \in \T \}$,
          where $S(t) = \{(x_r^i, y_r^i, k) \given k = t\}$.
          In this case, we set a constraint $B$ on the cost in each time.

    \item $\X_3 = \{\S \subseteq \V:  c(\S^i) \le B, \forall i \in \P \}$,
          where $S^i = \{(x_r^k, y_r^k, t) \given k = i\}$.
          Here, we set a constraint $B$
          on the cost used by all robots for sensing the location $(x^i, y^i)$.
\end{itemize}

\noindent
2) \emph{Homogeneous constraints}.

The general form is a matroid $\M = (\V, \I)$ with
\begin{equation*}
    \I = \lbbb \S \subseteq \V: \sum_{t \in \T} \one( |\S \cap \V_t| ) \le L \rbbb,
\end{equation*}
where $\one( |\S \cap \V_t| ) = 1$ if $ |\S \cap \V_t| \ge 1 $ and $\one( |\S \cap \V_t| ) = 0$ if $ |\S \cap \V_t| = 0$.
We refer to these constraints as homogeneous as we do not differentiate the behavior of different robots under these constraints. In this case, the general form above can be used in the following ways:
\begin{itemize}
    \item $\I_{21} = \{ \S \subseteq \V: |\S \cap \V_t| \le L(t)\}$. \\
          In this case, we can model the constraint on the number of deployed robots in each time $t \in \T$.

    \item $\I_{22} = \{ \S \subseteq \V: \one( |\S \cap \V_t| ) \le L(t) \}$. \\
          In this case, we can model the constraint on the deployment status in each time $t \in \T$. We use the deployment status to indicate if there is at least one deployment at time $t$. $L(t) \in \{0, 1\}$, where $L(t)=0$ means there is no deployment at $t$ and $L(t) = 1$ means there is at least one deployment at $t$. In this way, the constraint can be used to control intermittence by restricting times when there can be no deployment.

    \item $\I_{23} = \{\S \subseteq \V: \sum_{t \in \T} \one( |\S \cap \V_t| ) \le L \}$. \\
          This is the most general form in this case, as shown in the beginning. We can use this to model the constraint on the number of non-zero deployments from $t=1$ to $t=T$. If there is at least one deployment at some time, we call it a non-zero deployment. This constraint is also an intermittence constraint.
\end{itemize}

\noindent
3) \emph{Heterogeneous constraints}.

The general form is a matroid $\M = (\V, \I)$ with
\begin{equation*}
    \I = \lbbb \S \subseteq \V: \sum_{t \in \T} \one( |\S_r \cap \V_t| ) \le L_r, \forall r \in \R \rbbb,
\end{equation*}
where $S_r = \{(x_k^i, y_k^i, t)\;|\; k = r\}$. Also, $\one( |\S_r \cap \V_t| ) = 1$ if $ |\S_r \cap \V_t| \ge 1 $, and $\one( |\S_r \cap \V_t| ) = 0$ if $ |\S_r \cap \V_t| = 0$.
We refer to these constraints as heterogeneous as different robots will have different constraints that are unique. In this case, this general form can be used in the following ways:
\begin{itemize}
    \item $\I_{31} = \{ \S \subseteq \V: |\S_r \cap \V_t| \le L_r(t), \forall r \in \R\}$. \\
          In this case, we can model the constraint on the number of deployments for robot $r$ in each time $t$.

    \item $\I_{32} = \{ \S \subseteq \V: \one( |\S_r \cap \V_t| ) \le L_r(t), \forall r \in \R \}$. \\
          In this case, we can set a deployment status for robot $r$ in each time $t$. As before,  $L_{r}(t)=0$ means there is no deployment for robot $r$ at $t$ and $L_r(t) = 1$ means there is at least one deployment for robot $r$ at time $t$. Therefore, this constraint is especially useful for modeling any available/unavailable time window \cite{gini2017multi} for each robot $r \in \R$. Specifically, if the available deployment time for robot $r$ is time $t=1$ and $t=2$, we can set $L_{r}(1) = 1$, $L_{r}(2) = 1$, and $L_{r}(t) = 0, \forall t \in \T \setminus \{1, 2\}$. Clearly, this constraint can be classified as an intermittence constraint.

    \item $\I_{33} = \{\S \subseteq \V: \sum_{t \in \T} \one( |\S_r \cap \V_t| ) \le L_r, \forall r \in \R\}$. \\
          This is the most general form in this case as shown in the beginning. We can use this to model the constraint on the number of \emph{non-zero} deployments for robot $r$ from $t=1$ to $t=T$. This is also an intermittence constraint.
\end{itemize}

\begin{figure*}[!t]
    \centering
    \subfigure[Time $t = 1$.]{
        \label{fig: gp_1}
        \includegraphics[width = .31\linewidth]{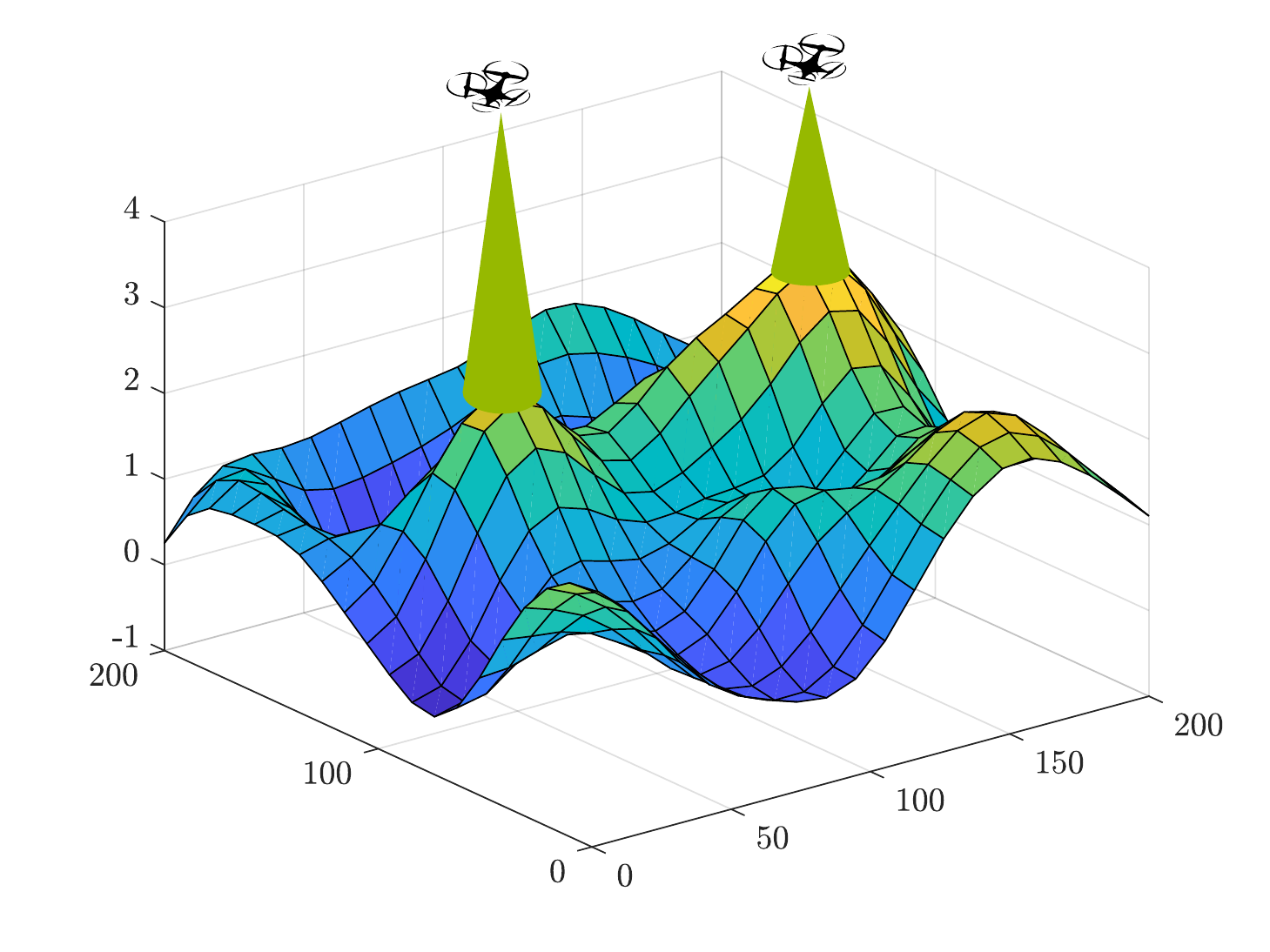}}
    \subfigure[Time $t = 2$.]{
        \label{fig: gp_2}
        \includegraphics[width = .31\linewidth]{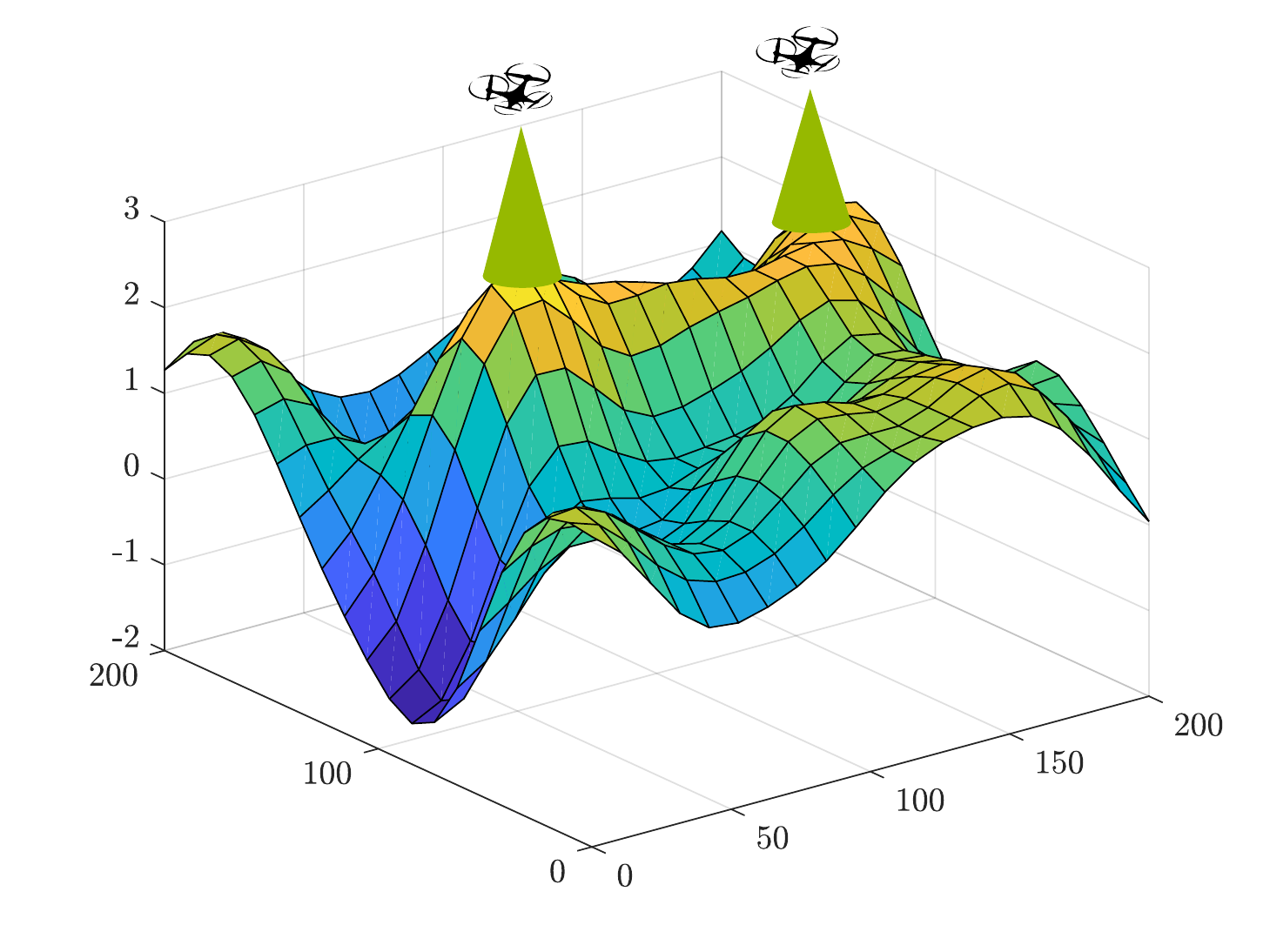}}
    \subfigure[Time $t = 5$.]{
        \label{fig: gp_5}
        \includegraphics[width = .31\linewidth]{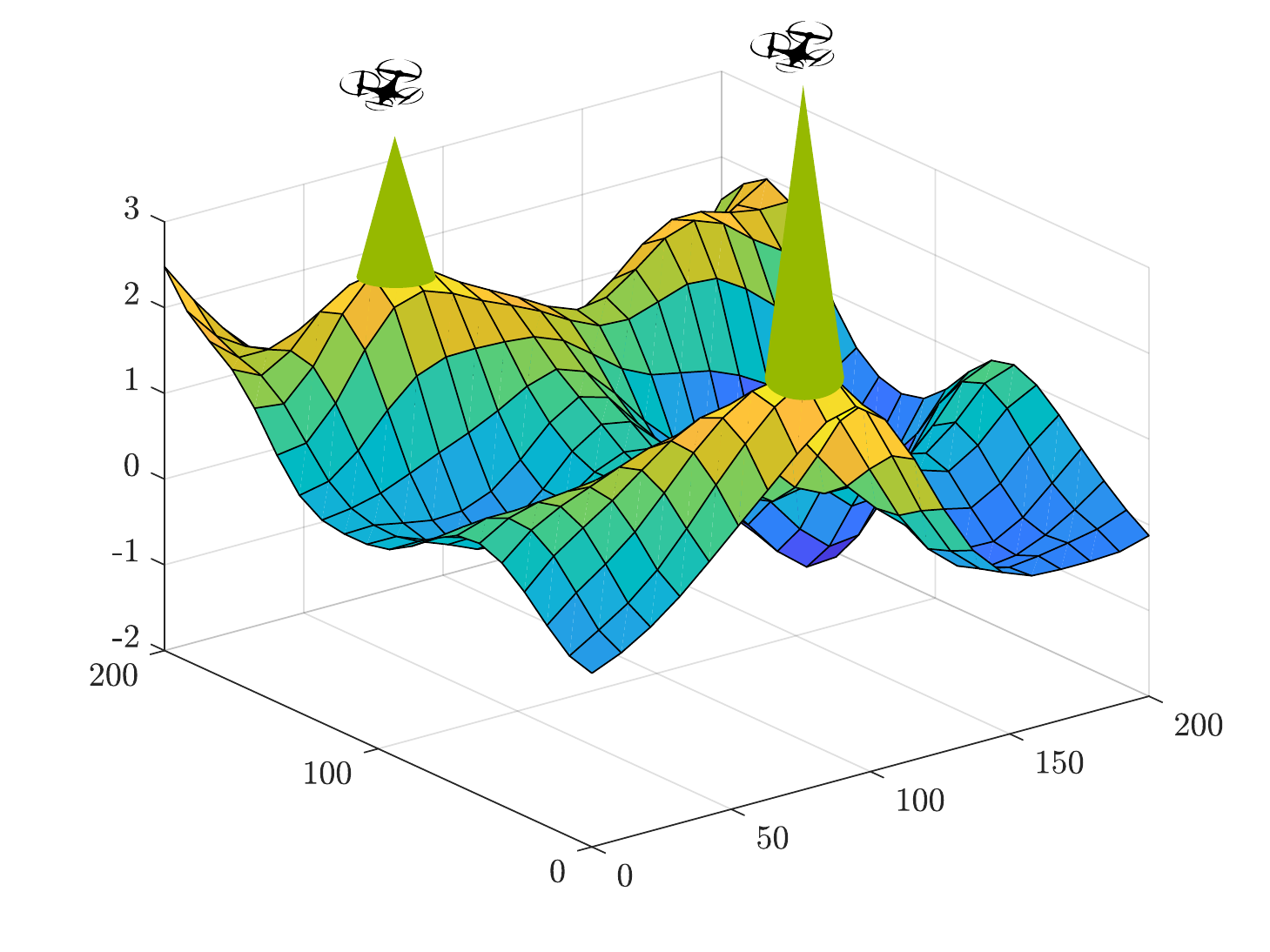}}
    \caption{An example of intermittent deployment with 2 robots and time horizon 5 with a solution as $\D = \{(x_r^i, y_r^i, t): r \in \{1, 2\}, i \in \P, t \in \{1, 2, 5\}\}$. An instance of this solution can be as follows. (a). At $t=1$, there are two deployments: $(x_1^{85}, y_1^{85}, 1)$ and $(x_2^{190}, y_2^{190}, 1)$. Also, the $85$th location corresponds to the actual 2D location $(95, 175)$ and the $190$th location corresponds to the actual 2D location $(190, 175)$ (b). At $t=2$, there are two deployments: $(x_1^{110}, y_1^{110}, 2)$ and $(x_2^{180}, y_2^{180}, 2)$. (c). At $t=5$, there are two deployments: $(x_1^{80}, y_1^{80}, 5)$ and $(x_2^{150}, y_2^{150}, 5)$.}
    \label{fig: intermittent_sensing_example}
\end{figure*}

An illustrative example of the intermittent deployment problem
is given in \figref{fig: intermittent_sensing_example}.

In the intermittent deployment problem, we desire high-quality predictions of our process, and thus we choose to maximize the mutual information we collect over deployments. We write the matroid intersection constraint as $\I = \{\S \subseteq \V: \S \in \bigcap_{k \in \K} \I_k\}$ where $\K$ represents the set indexing the matroids constraints. Also, we denote by $\X = \{\S \subseteq \V: S \in \bigcap_{\ell \in \L} \X_\ell\}$ and $\L$ represents the set indexing the knapsack constraints. Then, the general form of the intermittent deployment problem is now given as:
\begin{equation}
    \begin{split}
        \underset{\S \subseteq \V}{\text{maximize}} \quad & M(\S) \\
        \text{subject to} \quad & \S \in \X, \S \in \I,
    \end{split}
    \label{eq: formulation}
\end{equation}

\begin{remark}
    We have given general forms of various constraints in order to begin a library of models for the intermittent deployment problem that can meet various application requirements. In the simulation section, we will give an example to demonstrate this idea.
\end{remark}

\section{Solution Method and Analyses}
\label{sec: intermittent deployment}

\begin{theorem}
    Both the homogeneous constraints and heterogeneous constraints are matroidal.
\end{theorem}

The proof is omitted here due to space limitations.  The reader is referred to our previous work \cite{williams2017decentralized,liu2019submodular} for applicable proof methodologies.

\begin{theorem}
    Consider the intermittent deployment problem given by \eqref{eq: formulation}. If the 2D space is discretized into $P \times Q$ cells, then the cardinality of the solution $|\D|$ should satisfy $|\D| \le \frac{P \cdot Q \cdot R \cdot T}{2}$ to ensure the non-decreasing property of $M(\D)$. If the maximum cardinality of the solution $\D$ is $|\D|$, then the discretization should satisfy $P \cdot Q \ge \frac{2 |\D|}{R \cdot T}$ to ensure the non-decreasing property of $M(\D)$.
\end{theorem}

\begin{proof}
    When the ground set is $\V$, we have from \cite{krause2008near} that $M(\D)$ is monotonic non-decreasing when $|\D| \in (0, \frac{|\V|}{2}]$ and monotonic non-increasing when $|\D| \in [\frac{|\V|}{2}, |\V|]$. As the 2D space is discretized into $P \times Q$, the cardinality of the ground set $\V$ is $P \cdot Q \cdot R \cdot T$. Therefore, if $P \times Q$ is fixed, then $|\D| \le \frac{P \cdot Q \cdot R \cdot T}{2}$ can ensure the non-decreasing property of $M(\D)$. On the contrary, if $|\D|$ is fixed, i.e., the total number of deployments is fixed, a proper discretization of the 2D space, i.e., $P \cdot Q \ge \frac{2 |\D|}{R \cdot T}$, ensures the non-decreasing property.
\end{proof}

We here provide a general greedy method from \cite{badanidiyuru2014fast} to demonstrate how to solve the problem when considering $p$ matroid constraints and $\ell$ knapsack constraints.  We contextualize the algorithm to our problem for completeness and also tutorial value.

\begin{algorithm}[!t]
    \caption{The Intermittent Deployment Problem}
    \label{alg: 2}
    \textbf{Input:} The inputs are as follows:
    \begin{itemize}
        \item matroid intersection constraint $\bigcap_{i=1}^p \I_i$, knapsack constraint $\bigcap_{j=1}^\ell \X_j$; mutual information function $M(\cdot)$
    \end{itemize}

    \textbf{Output:} The deployment set $\D$.

    \begin{algorithmic}[1]
        \Statex
        \State $\T \leftarrow \emptyset$;
        \State $d \leftarrow \max_{e \in \V} M(\{e\}) $;
        \For{($\rho = \frac{d}{p+\ell}; \rho \le \frac{2d}{p+\ell} |\V|; \rho \leftarrow (1+\eta) \rho$)}
        \State $\S \leftarrow \emptyset$;
        \State $M_\rho \leftarrow \max_{e \in \V} \{M(\{e\}): \frac{M(\{e\})}{\sum_{j=1}^\ell c_{j}(e)} \ge \rho \}$;
        \State
        \For{($\tau = M_\rho; \tau \ge \frac{\eta}{|\V|} M_\rho; \tau \leftarrow \frac{\tau}{1 + \eta}$)}
        \For{$\forall e \in \V$}
        \If{$\S \cup \{e\} \in \I$}
        \State $m_e \leftarrow M(\{e\} \given \S)$;
        \If{$m_e \ge \tau$ and $\frac{m_e}{\sum_{j=1}^\ell c_{j}(e)} \ge \rho$}
        \State
        \If{$c_j(\S \cup \{e\}) \le B, j=1, \ldots, \ell$}
        \State $\S \leftarrow \S \cup \{e\}$;
        \Else
        \State $\T \leftarrow \T \cup \{e\}$;
        \State $\T \leftarrow \T \cup \{\S\}$;
        \State \textbf{restart} with the next $\rho$;
        \EndIf
        \State
        \EndIf
        \EndIf
        \EndFor
        \EndFor
        \State $\T \leftarrow \T \cup \{{\S}\}$;
        \EndFor
        \State $\D \leftarrow \argmax_{\S \in \T} M(\S)$.
    \end{algorithmic}
\end{algorithm}

The \algref{alg: 2} has two loops: the outer loop and the inner loop to deal with knapsack constraints and matroid constraints, respectively. In each iteration of the outer loop (Line 3-26), the algorithm picks a threshold $\rho$ as a marginal-to-cost ratio, i.e., $\frac{m_e}{\sum_{j=1}^\ell c_{j}(e)}, \forall e \in \V$. Again, $m_e = M(\{e\} \given \S)$ is the marginal value of $e$ under the current solution $\S$, and $e = (x_{r}^i, y_{r}^i, t)$ is one element of $\V$. Also, $c_{j}(e)$ is the cost of $e \in \V$ in $j$th knapsack constraint. The $\S$ chosen by the outer loop will serve as one candidate for the final solution.
The threshold $\rho$ ensures we do not exceed the budget constraint by discretizing the space of the marginal-to-cost ratio. The inner loop (Line 7-24) also has a threshold $\tau$ for the marginal value $m_e$ and decreases in every iteration. This threshold $\tau$ acts as a lower bound for the marginal value of $m_e$ in the current settings. Under different $\tau$'s, we can add different $e$'s to the current solution $\S$ when constraints are satisfied (Line 14). The final $\S$ forms one candidate for the problem solution. There might be a case that only the knapsack constraints are not satisfied. We then need to keep track of this case (Line 16-17) when deciding the final problem solution $\D$.

\begin{theorem}[\cite{badanidiyuru2014fast}]
    Let $\D$ be the greedy solution of the intermittent deployment. Also, denote by $\D^\star$ an optimal solution of \algref{alg: 2}, then we have an optimality bound $M(\D) \ge \frac{1}{(1+\eta)(p + 2 \ell + 1)} M(\D^\star)$ with $\eta \in (0, 1]$ .
    \label{thm: optimality_ratio}
\end{theorem}

\begin{theorem}
    The computational complexity of \algref{alg: 2} for solving the intermittent deployment problem is $\O(\frac{|\V|^4}{\eta^2} \log^2 \frac{|\V|}{\eta})$.
\end{theorem}

\begin{proof}
    For computing marginal gains of mutual information, it takes $\O(|\V|^3)$ runs. If we denote each call of calculating the mutual information gain as an oracle call, it takes $\O(\frac{|\V|}{\eta^2} \log^2 \frac{|\V|}{\eta})$ calls to get the final solution \cite{badanidiyuru2014fast}. Therefore, the computational complexity is $\O(\frac{|\V|^4}{\eta^2} \log^2 \frac{|\V|}{\eta})$.
\end{proof}

\section{Simulation Results}
\label{sec: sim}

In this section, we study a problem to compare the greedy solution with an optimal solution. The optimal solution is computed by enumerating all combinations of deployments. Consider a $P \times Q$ 2D space. The intermittent deployment we build here is one instance of our general form. Specifically, we use $\I_{21}$, $\I_{23}$, and the general knapsack constraints. Specifically, we put a constraint $L(t)$ on the number of deployed robots for $t \in \T$, a constraint $L$ on the total number of non-zero deployments, and a constraint on the cost of robots used in each time. We denote by $R$ the number of robots and set $L(t) = R, \forall t \in \T$. Finally, we get the number of combinations/cases that we need to consider when evaluating an optimal solution as $\binom{T}{1} \binom{R(PQ + 1)}{R} + \ldots + \binom{T}{L} \binom{R(PQ + 1)}{R}^L$. First, since there are $T$ times available and we can collect samples at a maximum of $L$ different times, there are $\binom{T}{\ell}, \forall 1 \le \ell \le L$ combinations. Then, if we can deploy robots at $\ell$ different times, we have $\binom{R(PQ + 1)}{R}^\ell$ combinations as there are $\binom{R(PQ + 1)}{R}$ combinations in each of $\ell$ times. $(PQ+1)$ is the number of combinations for every robot since every robot can be deployed to one of the $P \cdot Q$ positions and also can be not deployed. To make the problem tractable when computing an optimal solution, we simulate the problem with $\T \in \{4, \ldots, 8\}$, $L \in \{2, \ldots, 4\}$, $\R \in \{2, \ldots, 4\}$, $P \in \{3, \ldots, 5\}$ and $Q \in \{3, \ldots, 5\}$. For the environmental process, we use the SE kernel for both temporal and spatial kernels. The cost of every robot $r$ is the multiplication of the traveling cost with a weight factor $w_r$, which is a random number generated in $(0,1)$. In the simulation, we first collect training samples from an Gaussian mixture model (GMM) \cite{williams2017generalized}, and then train the GP using these samples. Finally, we make the deployment strategy based on the GP prediction. The 2D space is $200 \times 200$. The GMM is modeled as a linear combination of fixed basis functions and time-varying weights. In the simulation, we use 5 different basis functions. So, the output the GMM is $\varphi(x,y,t) = \sum_{i=1}^5 w_i(t) b_i(x,y) = \w(t)^\trs \b(x,y)$, where $w_i(t) \in \RR$ is a weight and $b_i(x,y) \in \RR$ is the output of a basis function. $\w(t) \in \RR^5$ is the stacked weights at time $t$ and $\b(x,y) \in \RR^5$ is the stacked basis functions for the location $(x,y)$. The dynamics of weights are $\dot{\w} = A \w + \sigma(t)$, where $A = -I_{5 \times 5}$ and $\sigma(t) \in \RR^5$ is a Gaussian noise. The initial settings of weights are $\{5,5,3,8,4\}$ and the initial locations of these basis functions are $(100, 100), (60, 80), (40, 30), (160, 160), (160, 30)$.

We run simulations 100 times under the settings described above. \figref{fig: ratio} shows the optimality ratios of the greedy solution to an optimal solution for different problem sizes. We denote the problem size as $(L \cdot R \cdot P \cdot Q)$. As we use two matroid constraints $\I_{21}$, $\I_{23}$ to constrain the time and a general knapsack constraint to constrain the cost, then $p=2$ and $\ell=1$. Also, we set the tunable parameter $\eta = 0.1$ in \algref{alg: 2}. Therefore, the optimality ratio of the greedy method is $\frac{1}{(1+\eta)(p + 2 \ell + 1)} = \frac{1}{(1+\eta)(2+2+1)} \approx 18\%$ according to \thmref{thm: optimality_ratio}. In \figref{fig: utility}, we show the mutual information of the greedy solution for different problem sizes. Since different simulations may have the same problem size, we here use the average of different utilities/mutual information if this case happens. From the result, we can see that as the set $\D$ becomes big, the mutual information become small. An intuitive example to show the connection between a deployment strategy and a GMM environment can be seen in \figref{fig: intermittent_sensing_example}. Any selected element from the ground set defines a sensing location and time, and all selected elements make the deployment strategy set. We also direct the reader to the attached video for deeper insight.

\begin{figure}[!tbp]
    \centering
    \includegraphics[width=0.8\linewidth]{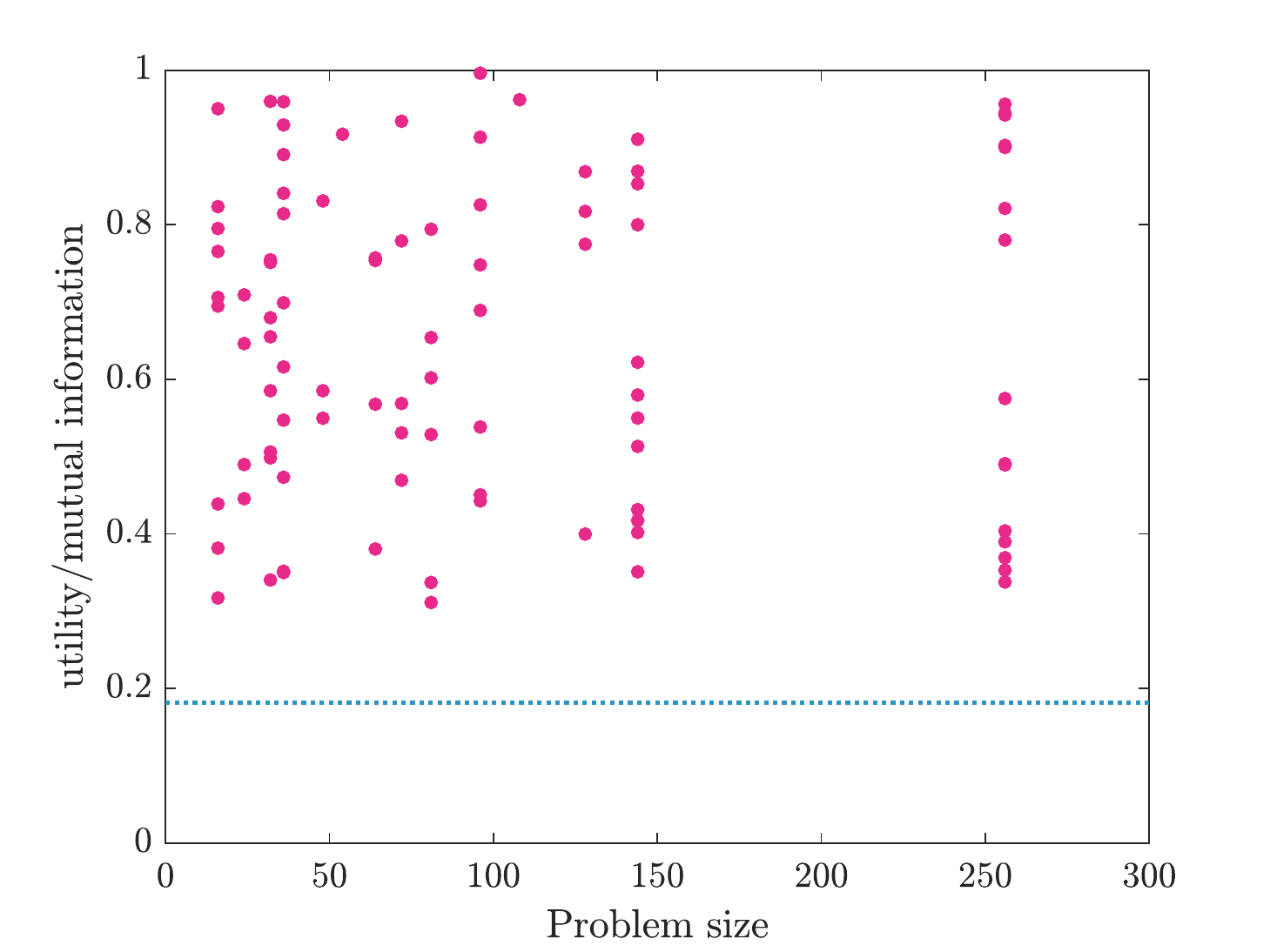}
    \caption{The optimality ratio of the greedy method with respect to different problem sizes.}
    \label{fig: ratio}
\end{figure}

\begin{figure}[!tbp]
    \centering
    \includegraphics[width=0.8\linewidth]{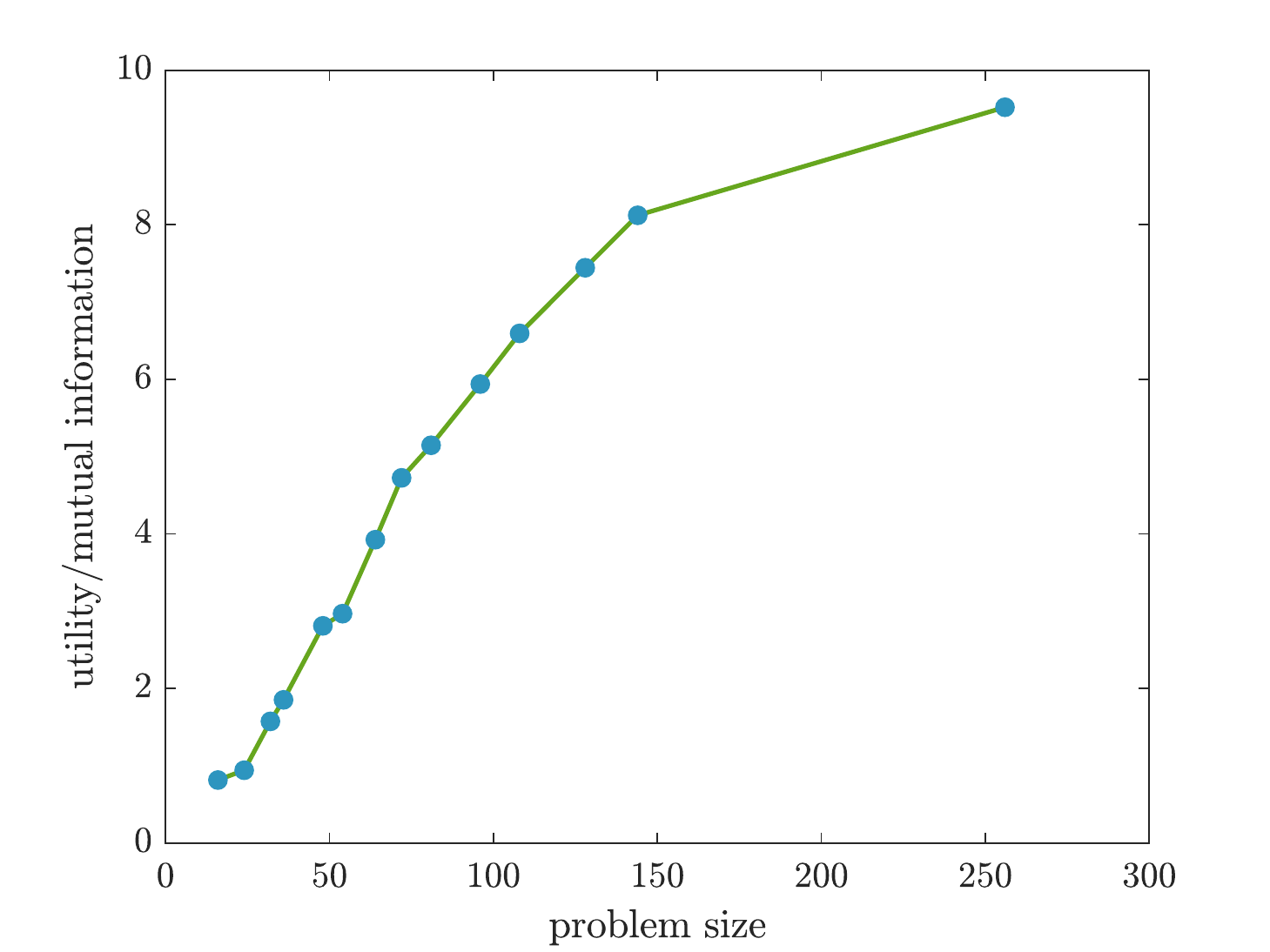}
    \caption{The objective values with respect to different problem sizes.}
    \label{fig: utility}
\end{figure}

\section{Conclusions and Future Work}
\label{sec: conclusion}

In this paper, we proposed a new intermittent deployment problem and demonstrated how to use matroid and knapsack constraints to model different instances of the problem. We also illustrated how to utilize a general submodular maximization method to solve our problem. Finally, the effectiveness of the solutions was demonstrated by Monte Carlo simulations. Directions for future work include performing field experiments with intermittent deployments and exploring the reduction of computational complexity.

\bibliography{ref}

\end{document}